%% file: main.tex
\documentclass{article}

\usepackage[final]{neurips_2023}

\usepackage[utf8]{inputenc} 
\usepackage[T1]{fontenc}    
\usepackage[colorlinks=true,linkcolor=blue,citecolor=blue,urlcolor=blue]{hyperref}      
\usepackage{url}            
\usepackage{booktabs}       
\usepackage{amsfonts}       
\usepackage{nicefrac}       
\usepackage{microtype}      

\usepackage{algorithm}
\usepackage{algpseudocode}

\input{math_commands}

\usepackage{amsmath}
\usepackage{amssymb}
\usepackage{amsthm}
\usepackage{multirow}
\usepackage{nolbreaks}
\usepackage{array}
\usepackage{graphics}
\usepackage{mathtools}
\usepackage{enumerate}
\usepackage{xspace}
\usepackage{wrapfig}
\usepackage{pifont}
\usepackage{float}
\usepackage{nicefrac}
\usepackage{bm}
\usepackage{booktabs}
\usepackage{multirow}
\usepackage{mathtools}
\usepackage{graphicx}
\usepackage[table,x11names,dvipsnames]{xcolor}

\newcommand{\te}[1]{{#1}}
\newcommand{\tr}[1]{{{#1}'}}
\newcommand{\idx}[2]{{#1}_{{#2}}}
\newcommand{\mybold}[1]{\textbf{#1}}
\newcommand{\mbf}[1]{{#1}}
\newcommand\mydots{\hbox to 1em{.\hss.\hss.}}
\newcommand{\spacer}{=}

\makeatletter
\newcommand{\shorteq}{%
  \settowidth{\@tempdima}{\spacer}
  \resizebox{\@tempdima}{\height}{=}%
}
\makeatother

\makeatletter
\newcommand{\shortneq}{%
  \settowidth{\@tempdima}{\spacer}
  \resizebox{\@tempdima}{\height}{\neq}%
}
\makeatother

\makeatletter
\newcommand{\shortsub}{%
  \settowidth{\@tempdima}{\spacer}
  \resizebox{\@tempdima}{\height}{-}%
}
\makeatother

\definecolor{red}{RGB}{215,48,39}
\definecolor{green}{RGB}{26,152,80}
\definecolor{lightgray}{gray}{0.96}
\definecolor{blue}{RGB}{30, 144, 255}
\theoremstyle{definition}

\newtheorem{lemma}{Lemma}

\usepackage[toc,page,header]{appendix} 
\usepackage{minitoc}

\title{Error Discovery By Clustering Influence Embeddings}

\author{%
Fulton Wang\thanks{The two authors contributed equally. Correspondence to fultonwang@meta.com} \\
Meta\\
   \And Julius Adebayo$^*$\\
   Prescient Design / Genentech\\
   \And Sarah Tan \\
   Cornell University \\
   \And Diego Garcia-Olano \\
   Meta\\
   \And Narine Kokhlikyan\\
   Meta\\
}

\begin{document}

\maketitle

\begin{abstract}
We present a method for identifying groups of test examples---slices---on which a model under-performs, a task now known as \textit{slice discovery}. We formalize \textit{coherence}---a requirement that erroneous predictions, within a slice, should be \textit{wrong for the same reason}---as a key property that any slice discovery method should satisfy.  We then use influence functions to derive a new slice discovery method, \texttt{InfEmbed}, which satisfies coherence by returning slices whose examples are influenced similarly by the training data.  \texttt{InfEmbed} is simple, and consists of applying K-Means clustering to a novel representation we deem \emph{influence embeddings}. We show \texttt{InfEmbed} outperforms current state-of-the-art methods on 2 benchmarks, and is effective for model debugging across several case studies.\footnote{Code to replicate our findings is available
at: \url{https://github.com/adebayoj/infembed}}
\end{abstract}

\input{sections/introduction}
\input{sections/background}
\input{sections/influenceembed}
\input{sections/experiments}
\input{sections/relatedwork.tex}
\input{sections/conclusion}

\begin{ack}
Fulton Wang, Sarah Tan, Diego Garcia-Olano and Narine Kokhlikyan were employed by Meta, while Julius Adebayo was employed by Genentech over the course of work on this project. No external funding was received or used by the authors over the course of the project.
\end{ack}

\newpage
\bibliography{ref}
\bibliographystyle{plainnat}
\appendix
\input{sections/appendix.tex}

\end{document}

%% file: math_commands.tex
\usepackage{amsmath,amsfonts,bm}

\def\eqref#1{equation~\ref{#1}}

\def\1{\bm{1}}

\DeclareMathAlphabet{\mathsfit}{\encodingdefault}{\sfdefault}{m}{sl}
\SetMathAlphabet{\mathsfit}{bold}{\encodingdefault}{\sfdefault}{bx}{n}



%% file: sections/introduction.tex
\section{Introduction}
Error analysis is a longstanding challenge in machine learning~\citep{pope1976statistics, amershi2015modeltracker, sculley2015hidden, chung2019automated, chakarov2016debugging, cadamuro2016debugging,  kearns2018preventing, kim2019multiaccuracy, zinkevichrules2020}. Recently, \citet{eyuboglu2022domino} formalized a type of error analysis termed the \emph{slice discovery problem}. In the \emph{slice discovery problem}, given a multi-class classification model and test data, the goal is to partition the \emph{test} data into a set of \emph{slices}---groups of test examples---by model performance.  An effective slice discovery method (SDM) should satisfy two desiderata:
\vspace{-2mm}
\begin{enumerate}
\itemsep0em
\item \emph{Error surfacing}: Identify \emph{under-performing slices}, i.e. slices with low accuracy, if they exist.
\item \emph{Coherence}: Return slices that are \emph{coherent} in the following sense: Erroneous predictions in a given slice should have the same \emph{root cause}, i.e. be ``wrong for the same reason''.
\end{enumerate}
 To illustrate, suppose we have a pre-trained model that detects hip fractures from an x-ray, trained on data from one hospital, that is now deployed in a new hospital.  The model might err frequently on two groups: 1) examples containing a spurious signal like the frequency signature of the scanner from the first hospital \citep{badgeley2019deep}, and 2) examples from patients not represented in the new hospital.  An effective SDM should uncover these two groups. The slice discovery problem has taken on renewed importance given that models trained via empirical risk minimization exhibit degraded performance on certain data groups---slices---due to distribution shift~\citep{koh2021wilds}, noisy labels~\citep{northcutt2021pervasive}, and imbalanced data~\citep{idrissi2022simple}. 

In developing a SDM, there are two key challenges. First, \emph{formalizing the coherence property}, i.e., defining the \emph{root cause} of an error, identifying whether two errors have similar root causes, and ultimately, defining whether a slice is \emph{coherent}. Second, \emph{identifying the representation to use}: returning to the x-ray example, there is typically no additional metadata that indicates whether an x-ray contains the spuriously correlated scanner pattern. The input to a SDM will typically be the x-ray or a representation thereof, without additional metadata that indicates which features are spurious. Thus, identifying a representation that can help distinguish under-performing groups from high performing ones while maintaining coherence is a key challenge.

\begin{figure*}[t]
\centering
\includegraphics[page=1, scale=0.69]{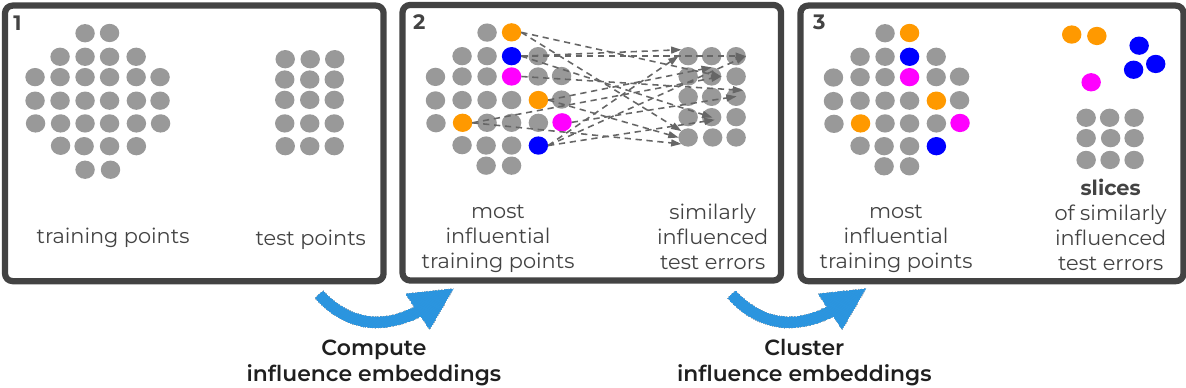}
\label{fig:infembed_high}
\caption{Schematic of the \texttt{InfEmbed} slice discovery method. In slice discovery, the goal is to partition the test points into groups---slices---by model performance. First, for each test point, we compute its influence embeddings---a low dimensional representation of the vector of influence scores of the training set. In the second stage, we cluster this influence embedding set to group points by training set influence.}
\vspace{-11pt}
\end{figure*}

\textbf{Influence Explanations}: To formalize the coherence property, we appeal to influence functions~\citep{koh2017understanding}, which tractably estimate the ``effect'' of a training example on a test example---how much the loss of the test example would change if the model were re-trained without that training example. We define the \emph{influence explanation} of a given test example to be the vector of the `influence' of each training example on the test example. Since a test example's influence explanation is the ``cause'' of its prediction from a training dataset perspective, we say a SDM achieves \emph{coherence} if the slices it returns have similar influence explanations.

\textbf{Influence Embeddings}: 
To create coherent slices, one might naively cluster test examples' \emph{influence explanations}.  
However, each influence explanation is high dimensional---equal to the size of the training dataset.  
Instead, we transform influence explanations into \emph{influence embeddings}. 
We define influence embeddings to be vectors with the \emph{dot-product property}: that the influence of a training example on a test example is the dot-product of their respective embeddings.

\textbf{Clustering Influence Embeddings}:
Given the coherence-inducing properties of influence embeddings, we propose \texttt{InfEmbed}, an SDM that applies K-means clustering to the influence embeddings of all test examples.
Despite its simplicity, $\texttt{InfEmbed}$ satisfies the two aforementioned desiderata.
First, it is error-surfacing, as mis-predicted examples that are ``wrong for the same reasons'', i.e. have similar influence explanations, tend to be grouped into the same slice. 
Second, $\texttt{InfEmbed}$ satisfies several desirable properties.
It produces coherent slices, and can accomodate user-defined choices regarding the subset of a model's representations to use, ranging from all to only the last layer~\citep{d2022spotlight, sohoni2020no}. 
In addition, the slices produced by $\texttt{InfEmbed}$ tend to have \emph{label homogeniety}, i.e. examples are homogeneous in terms of the true and predicted label. 
Unlike other SDMs, these properties are implicit consequences the rigorous influence function-based derivation.

 \textbf{Contributions}: We summarize our contributions as follows:
 \vspace{-3pt}
 \begin{enumerate}
     \item \textbf{New slice discovery method}: We propose $\texttt{InfEmbed}$, a SDM to identify under-performing groups of data that applies K-Means to a new representation: \emph{influence embeddings}.
     
     \item \textbf{Coherence \& Theoretical Justification.} We formalize the \textit{coherence} property, which requires that erroneous predictions within a slice should be ``wrong for the same reason''.  We then define influence embeddings, and prove that a procedure which clusters them achieves coherence: it returns slices whose examples are influenced similarly by the training data.

     \item \textbf{Empirical Performance.} We show that $\texttt{InfEmbed}$ outperforms previous state-of-the-art SDMs on 2 benchmarks---DCBench ~\citep{eyuboglu2022dcbench}, and SpotCheck~\citep{plumb2022evaluating}. $\texttt{InfEmbed}$ is scalable, able to recover known errors in variety of settings, and identifies coherent slices on realistic models.
 \end{enumerate}

%% file: sections/background.tex
\section{Background\label{sec:background}}
At a high level, the goal of the slice discovery problem is to partition a test dataset into slices--groups of data points---such that some slices are \emph{under-performing}, i.e. have low accuracy, and slices are \emph{coherent}, i.e. erroneous predictions in a given slice are ``wrong for the same reason''.  
We first formally define the slice discovery problem, and then overview influence functions, which we use to define what it means for a slice to be coherent - one of our key contributions, and which will be detailed in Section \ref{sec:influenceembeddings}.

\textbf{Slice Discovery Problem:} 
Given a trained \emph{multi-class} classification model $f(\boldsymbol{\cdot};\theta)$ over $C$ classes with model parameters $\theta$ and an example $x$ from some input space ${X}$, $f(x;\theta)\in \mathbb{R}^C$ is the prediction for the example and $f(x;\theta)_c$ is the pre-softmax prediction for class $c$.  Given a test dataset $\te{\mybold{Z}}\coloneqq[\idx{\te{z}}{1},\mydots,\idx{\te{z}}{N}]$ with $N$ examples  where $\idx{\te{z}}{i}\coloneqq (\idx{\te{x}}{i},\idx{\te{y}}{i})$ and $\idx{\te{x}}{i}\in\mathcal{X}$, $\idx{\te{y}}{i}\in \mathcal{Y}\coloneqq[0,1]^{C}$, the goal of the slice discovery problem is to partition the test dataset into $K$ slices: $\{\Phi_k\}_{k=1}^K$, where $\Phi_k \subseteq [N]$, $\Phi_k \cap \Phi_{k'}=\emptyset$ for $k\neq k'$, $\cup_{k=1}^K \Phi_k = [N].$  As we will make precise shortly, each slice $\Phi_k$ should be coherent, and some slices should be under-performing slices. The model is assumed to be trained on a training dataset with $\tr{N}$ examples, $\tr{\mybold{Z}}\coloneqq[\idx{\tr{z}}{1},\mydots,\idx{\tr{z}}{\tr{N}}]$ where $\idx{\tr{z}}{i}\coloneqq (\idx{\tr{x}}{i},\idx{\tr{y}}{i})$ with $\idx{\tr{x}}{i}\in\mathcal{X}, \idx{\tr{y}}{i}\in\mathcal{Y}$, so that $\theta = \operatorname{argmin}_{\theta'}\tfrac{1}{\tr{N}}\sum_{i\shorteq1}^{\tr{N}}L(\idx{\tr{z}}{i}; \theta')$. We assume the loss function $L$ is cross-entropy loss.  Thus, given test point $z=(x,y)$ with $x\in\mathcal{X}, y=[y_1,\mydots,y_C]\in \mathcal{Y}$, $L(z;\theta) = -\textstyle \sum_c y_c \log p_c$, with $p=[p_1,\mydots,p_C]$ and $p_c\coloneqq \tfrac{\exp(f(x;\theta)_c)}{\sum_{c'} \exp(f(x;\theta)_{c'}))}$, i.e. $p=\operatorname{softmax}(f(x;\theta))$.

\textbf{Influence Functions}: Influence functions estimate the \emph{effect} of a given training example,$z'$, on a test example, $z$, for a pre-trained model. 
Specifically, the influence function approximates the change in loss for a given test example $z$ when a given training example $z'$ is removed from the training data and the model is retrained.  
\citet{koh2017understanding} derive the aforementioned influence to be
$I(\tr{z},\te{z}) \coloneqq \nabla_{\theta} L(\tr{z};\theta)^{\intercal} H_{\theta}^{-1} \nabla_{\theta} L(\te{z};\theta),$ where $H_{\theta}$ is the loss Hessian for the pre-trained model: $H_{\theta}\coloneqq 1/n \sum_{i=1}^{n} \nabla^2_\theta L(z;\theta)$, evaluated at the pre-trained model's final parameter checkpoint. 
The loss Hessian is typically estimated with a random mini-batch of data.

The main challenge in computing influence is that it is impractical to explicitly form $H_{\theta}$ unless the model is small, or if one only considers parameters in a few layers. \citet{schioppa2022scaling} address this problem by forming a low-rank approximation of $H_{\theta}^{-1}$ via a procedure that does not explicitly form $H_{\theta}$. In brief, they run the Arnoldi iteration \citep{trefethen1997numerical} for $P$ iterations to get a $P$-dimensional Krylov subspace for $H_{\theta}$, which requires only $P$ Hessian-vector products.  Then, they find a rank-$D$ approximation of the restriction of $H_{\theta}$ to the Krylov subspace, a small $P\times P$ matrix, via eigendecomposition.  Their algorithm, $\operatorname{FactorHessian}$ (see Appendix \ref{appendix:lowinf} for details), which we use in our formulation, returns factors of a low-rank approximation of $H_{\theta}^{-1}$
\begin{align}
M, \lambda = \operatorname{FactorHessian}(\tr{\mybold{Z}}, \Theta, P, D)\ \text{where}\label{eq:factors}\\ \hat{H}_{\theta}^{-1} \coloneqq M\lambda^{-1}M^{\intercal} \approx H_{\theta}^{-1}, 
\label{eq:H_hat}
\end{align}
$M \in \mathbb{R}^{|\theta|\times D}$, $\lambda \in \mathbb{R}^{D\times D}$ is a \emph{diagonal} matrix, $|\theta|$ is the parameter count, $\hat{H}_{\theta}^{-1}$ approximates $H_{\theta}^{-1}$, $D$ is the rank of the approximation, $P$ is the Arnoldi dimension, i.e. number of Arnoldi iterations, and here and everywhere, \emph{configuration} $\Theta\coloneqq(L,f,\theta)$ denotes the loss function, model and parameters.  $\hat{H}_{\theta}^{-1}$ is then used to define the \emph{practical influence} of training example $\tr{z}$ on test \mbox{example $\te{z}$:}
\begin{align}
\hat{I}(\tr{z},\te{z}) &\coloneqq \nabla_{\theta} L(\tr{z};\theta)^{\intercal} \hat{H}_{\theta}^{-1} \nabla_{\theta} L(\te{z};\theta).\label{eq:influence_hat}
\end{align}

%% file: sections/influenceembed.tex
\section{Error Discovery By Clustering Influence Embeddings}
\label{sec:influenceembeddings}
We propose a slice discovery method (SDM) whose key contribution is to require returned slices have \emph{coherence}, and unlike past works, rigorously define coherence, which we do using influence functions.  In this section, we derive our SDM as follows: 1) We formally define that a slice has coherence if its examples have similar \emph{influence explanations}---which we define using influence functions, and leverage them to formalize a clustering problem.  2) Because influence explanations are high-dimensional, we derive \emph{influence embeddings}, a low-dimensional similarity-preserving approximation.  3) Using influence embeddings, we give a simple and efficient procedure to solve the clustering problem to give our SDM, \texttt{InfEmbed}. 4) We then describe an extension, \texttt{InfEmbed-Rule}, which instead of returning a partition with a user-specified number of slices, returns the largest slices satisfying user-specified rules.  5) We then show that the proposed formulation implicitly encourages a property, label homogeneity, that previous methods needed to explicitly encourage, in a manner that requires specifying a hard-to-tune hyperparameter.  6) Lastly, explain slices using training samples.

\subsection{Influence Explanations and Problem Formulation}
We seek to partition the test data into slices such that predictions for examples in the same slice have similar root causes as quantified by influence functions. Therefore, we propose to represent the root cause of the prediction for a test example, $\te{z}$, with its \emph{influence explanation $E(\te{z})$:}
\begin{align}
E(\te{z}) \coloneqq \{\hat{I}(\idx{\tr{z}}{j},z)\}_{j=1}^{\tr{N}},
\end{align} 
the vector containing the influence of every training example on the test example.  Thus, test examples with similar influence explanations are ``influenced similarly by the training data".  We then want each slice be \emph{coherent} in the following sense: the Euclidean distance between influence explanations of examples in the same slice should be low.

\begin{algorithm}

\begin{algorithmic}[1]
\caption{Algorithm for Discovering Problematic Slices}
\label{alg:basic}
\Procedure{GetEmbeddings}{$\te{\mybold{Z}},\tr{\mybold{Z}},\Theta,P, D$}
\State \textbf{Outputs}: \mbox{$D$-dimensional embeddings for test data $\te{\mybold{Z}}$}
\State $M,\lambda \gets \operatorname{FactorHessian}(\tr{\mybold{Z}}, \Theta, P, D)$\ \Comment{see Equation \ref{eq:H_hat}}
\State $\mu_i \gets \lambda^{-1/2} M^{\intercal} \nabla_\theta L(z_i;\theta)$ for $i=1,\mydots,N$ \Comment{see Equation \ref{eq:emb}}
\State $\boldsymbol{\mu} \gets [\mu_1,\mydots,\mu_N]$
\State \textbf{Return:} $\boldsymbol{\mu}$
\EndProcedure

\Procedure{InfEmbed}{K, $\te{\mybold{Z}},\tr{\mybold{Z}},\Theta,P, D$}
\State \textbf{Inputs}: number of slices $K$, training dataset $\tr{\mybold{Z}}$, test dataset $\te{\mybold{Z}}$, configuration $\Theta$, Arnoldi dimension $P$, influence embedding dimension $D$
\State \textbf{Outputs}: partition of test dataset into slices, $\Phi$
\State $\boldsymbol{\mu}\gets \operatorname{GetEmbeddings}(\te{\mybold{Z}},\tr{\mybold{Z}},\Theta,P, D)$
\State $[r_1,\mydots,r_{\te{N}}] \gets  \operatorname{K-Means}(\boldsymbol{\mu}, K)$
\Comment{compute cluster assignments for all $N$ examples}
\State \textbf{Return:} $\{\{i \in [N]: r_i=k\}\ \text{for}\ k\in[K]\}$ \Comment{convert cluster assignments to a partition}
\EndProcedure
\caption{Our SDM, \texttt{InfEmbed}, applies K-Means to influence embeddings of test examples.}
\end{algorithmic}

\end{algorithm}
\vspace{-3pt}

Therefore, we solve the slice discovery problem described in Section \ref{sec:background} via the following formulation---return a partition of the test dataset into slices, $\{\Phi_k\}_{k=1}^K$, that minimizes the total Euclidean distance between influence explanations of examples in the same slice. Formally, we seek:
\begin{align}
\hspace{-4pt}\operatorname{argmin}_{\{\Phi_k\}_{k=1}^K} \textstyle\sum_k \textstyle\sum_{i,i'\in \Phi_k} ||E(\idx{\te{z}}{i}) - E(\idx{\te{z}}{i'})||^2.\label{eq:formulation}
\end{align}

\subsection{Influence Embeddings to approximate Influence Explanations}
\label{sec:influence_embedding}
To form coherent slices, a naive approach would be to form the influence explanation of each test example, and apply K-Means clustering with Euclidean distance to them.  However, influence explanations are high-dimensional, with dimensionality equal to the size of the training data, $\tr{N}$.  Instead, we will use \emph{influence embeddings}: vectors such that the practical influence of a training example on a test example is the dot-product of their respective influence embeddings.  Looking at Equations \ref{eq:H_hat} and \ref{eq:influence_hat}, we see the influence embedding $\mu(z)$ of an example $z$ must be defined as
\begin{align}
\mu(z) \coloneqq \lambda^{-1/2} M^{\intercal}  \nabla_{\theta} L(z;\theta) \label{eq:emb}
\end{align}
and $M$, $\lambda$ are as defined in Section \ref{sec:background}, because for any training example $\tr{z}$ and test example $\te{z}$, $\hat{I}(\tr{z},\te{z})=\mu(\tr{z})^{\intercal}\mu(\te{z})$.  Procedure \texttt{GetEmbeddings} of Algorithm \ref{alg:basic} finds $D$-dimensional influence embeddings for the test dataset $\te{\mybold{Z}}$, which has runtime \emph{linear} in the Arnoldi dimension $P$.  The rank $D$ of the factors from $\texttt{FactorHessian}$ determines the dimension of the influence embeddings.

Influence embeddings satisfy a critical property---that if two examples have similar influence embeddings, they also tend to have similar influence explanations.  This is formalized by the following lemma, whose proof follows from the Cauchy-Schwartz inequality, and that influence is the dot-product of influence embeddings (see Appendix Section \ref{appendix:propoppo}) 

\textbf{Lemma 1:} There is a constant $C > 0$ such that for any test examples $\idx{\te{z}}{i},\idx{\te{z}}{j}$, $||E(\idx{\te{z}}{i}) - E(\idx{\te{z}}{j})||^2 \leq C||\mu(\idx{\te{z}}{i}) - \mu(\idx{\te{z}}{j}))||^2$.

\subsection{\texttt{InfEmbed}: Discovering Problematic Slices by Clustering Influence Embeddings}
We will solve the formulation of Equation \ref{eq:formulation} via a simple procedure which applies K-Means with Euclidean distance to the influence \emph{embeddings} of the test dataset, $\boldsymbol{\mu}\coloneqq[\mu(\te{z}_1),\mydots,\mu(\te{z}_{\te{N}})]$.  This procedure is justified as follows: Using Lemma 1,
given any partition $\{\Phi_k\}$, we know
\begin{align}
\textstyle\sum_k \textstyle\sum_{i,i'\in \Phi_k} ||E(\idx{\te{z}}{i}) - E(\idx{\te{z}}{i'})||^2 \leq C \textstyle\sum_k \textstyle\sum_{i,i'\in \Phi_k} ||\mu(\idx{\te{z}}{i}) - \mu(\idx{\te{z}}{i'})||^2\label{eq:rhs:main}.
\end{align}
However, the quantity in the right of Equation \ref{eq:rhs:main} is the \emph{surrogate} objective minimized by this procedure---it is the K-Means objective applied to influence embeddings, scaled by $C>0$ which does not matter.  
The quantity in the left of Equation \ref{eq:rhs:main} is the \emph{actual} objective which the formulation of Equation \ref{eq:formulation} minimizes.  
Therefore, the procedure is minimizing a surrogate objective which upper bounds the actual objective we care about.  
Put another way, by applying K-Means to influence embeddings, examples within the same slice will not only have similar influence embeddings, but also similar influence explanations as desired, by Lemma 1. 
Algorithm \ref{alg:basic} describes $\texttt{InfEmbed}$.
In addition, we find that normalizing the centroid centers in each iteration of K-Means is helpful, perhaps due to lessening the effect of outlier influence embeddings.
Overall, instead of K-means, other clustering algorithms can used as part of formulation.

\textbf{Computational Complexity}: We now take a step back to examine the computational complexity of the proposed procedure. The main computational bottleneck is the implicit Hessian estimation required to compute influence embeddings.~\citet{schioppa2022scaling}'s approach, which we rely on, has complexity $\mathcal{O}$($P$) where $P$ is the Arnoldi dimension. This is because each of the $P$ steps in the Arnoldi iteration requires computing a Hessian-vector product. 
Consequently, the complexity of computing influence embeddings is exactly the same as that of the influence function (IF) estimation. 
For the K-means portion, the complexity is $\mathcal{O}$(number of samples$\times$
number of k-means iterations $\times$
number of clusters). In practice, the clustering step is near instantaneous, so the method is dominated by implicit Hessian estimation step.

\subsection{\texttt{InfEmbed-Rule}: Finding Slices Satisfying a Rule}
\label{subsec:infembedrule}
The key hyper-parameter of the $\operatorname{InfEmbed}$ method is $K$, the number of slices to return.  In practice, it may not be intuitive for a user to choose $K$.  Instead, the user may want to know if there exist any coherent slices that are problematic, where `problematic' is pre-specified by a rule, e.g., \textit{a slice with greater than T samples and that has accuracy less than A}. Without specifying the rule, a practitioner has to iterate through all the returned slices to figure out which one is problematic. Therefore we propose a procedure that recursively clusters influence embeddings until slices satisfying a pre-specified rule are found, or until the slices are too small.  

The proposed approach, \texttt{InfEmbed-Rule},  is analogous to building a tree to identify slices satisfying the rule, where the splits are determined by K-Means clustering of influence embeddings.  In addition to letting the user specify more intuitive hyperparameters, this procedure also has the advantage that if a large slice with sufficiently low accuracy is found, it will not be clustered further.  Appendix Algorithm \ref{alg:rule} outlines this \texttt{InfEmbed-Rule} method.  Its inputs are the same as \texttt{InfEmbed}, except instead of $K$, 
one specifies accuracy threshold $A$, size threshold $S$, and branching factor $B$, which sets how many clusters the K-means call in each step of the recursion should return. It outputs a set of slices each with accuracy less than $A$ and size greater than $S$.

\subsection{Properties of Discovered Slices}
\label{sec:properties}
We now show that the slices discovered by \texttt{InfEmbed} possess properties similar to past SDM's.  To derive these properties, we consider what factors result in two examples $z=(x,y)$, $z'=(x',y')$ having influence embeddings with low Euclidean distance, so that K-Means would place them in the same slice. 
 To simplify analysis, we note that two examples would be placed in the same slice if their \emph{gradients} have high dot-product, because influence embeddings are linearly transformed gradients (see Equation \ref{eq:emb}) and Euclidean distance is equal to their negative dot-product plus a constant.  For further simplicity, we consider the case when the model $f$ is a linear model, i.e. when only gradients in the last fully-connected layer are considered when computing influence embeddings.
 Then, a straightforward computation shows that two examples $z,z'$ will be placed in the same slice if
\begin{align}
\nabla_{\theta}L(z;\theta)^{\intercal}\nabla_{\theta}L(z';\theta) = (y - p)^{\intercal}(y'-p') x^{\intercal}x',\label{eq:truth}
\end{align}
is high, where $L$ is cross-entropy loss, $\theta$ are the parameters, and for example $z$, $p\coloneqq[p_1,\mydots,p_C]=\operatorname{softmax}(f(x))$ is the predicted probabilities for each class, $y$ is the one-hot encoded label, $x$ is its last-layer representation, and $p',y',x'$ are defined analogously for example $z'$.

Looking at the $x^{\intercal}x'$ term of Equation \ref{eq:truth}, we see that two examples will be placed in the same slice if, all else being equal, their last-layer representations are similar as measured by the dot-product.  Past SDM's \citep{d2022spotlight,sohoni2020no} also use \emph{heuristic} similarity measures involving last-layer representations to decide if two examples should be in the same slice.  On the other hand, our influence function-based derivation shows that their dot-product is the ``canonical'' similarity measure to use.  Note that since influence embeddings consider gradients in layers beyond the last, they are able to consider information \emph{beyond} the last-layer representations, \emph{generalizing} past SDM's.

Looking at the $(y - p)^{\intercal}(y'-p')$ term of Equation \ref{eq:truth}, we see that two examples will be placed in the same slice if, all else being equal, their \emph{margins} are similar, as measured by dot-product.  The margin of an example, eg. $(y - p)$ or $(y'-p')$, is the vector containing the difference between the true label and predicted probability for each class.  Thus, examples with similar labels and predictions will tend to be in the same slice, i.e. slices tend to have \emph{label homogeniety}.  Past SDM's \citep{eyuboglu2022domino} also produce slices that tend to have label homogeniety, but do so through an ``error-aware'' mixture model that requires choosing a hard-to-tune tradeoff parameter.  On the other hand, our influence function-based derivation shows that the dot-product of margins is the ``canonical'' way to encourage label homogeniety, and does not require tuning a trade-off constant.  Note that label homogeniety is \emph{not} the only factor in consideration when forming slices.

\subsection{Slice Explanation via Slice Opponents}
\label{sec:method:proponents}
Given an under-performing slice, we are interested in identifying the root-cause of the erroneous predictions in it. We therefore compute the top-$k$ slice opponents and consider them as the root-cause. These are the $k$ training examples whose influence on the slice are most harmful, i.e. whose influence on the total loss over all examples in the slice is the most negative.

Formally, given a slice $\Phi$, we first compute the influence embeddings of the training data, $[\mu(z'_1),\dots,\mu(z'_{N'})]$.  Then, we compute the sum of the influence embeddings of the slice, $v\coloneqq \sum_{i\in\Phi}\mu({z_i})$.  Finally, due to the properties of influence embeddings, the top-$k$ slice opponents of the given slice, $O$, are the $k$ training examples $z'$ for which $v^{\intercal}\mu({z'})$ is the most negative, i.e. $O = \operatorname{argmin}_{O'\subseteq [N'], |O'|=k} v^{\intercal}(\sum_{i\in O'}\mu(z'_{i'}))$.  Inspecting slice opponents can provide insight into the key features responsible for low performance in a slice.


%% file: sections/experiments.tex
\section{Results} 
Here, we perform a quantitative evaluation of \texttt{InfEmbed} on the \texttt{dcbench} benchmark, and also perform several case studies showing \texttt{InfEmbed}'s usefulness. Please see Appendix Section \ref{sec:covid} on a case study applying \texttt{InfEmbed} to a small dataset where additional metadata can help explain slices.  For scalability, to compute influence embeddings, at times we will only consider gradients in some layers of the model when calculating the gradient $\nabla_{\theta}L(z;\theta)$ and inverse-Hessian factors $M, \lambda^{1/2}$ in Equation \ref{eq:emb}.  Furthermore, following \cite{schioppa2022scaling}, we do not compute those factors using the entire training data, i.e. we pass in a subset of the training data to \texttt{FactorHessian}.  For all experiments, we use Arnoldi dimention $P=500$, and influence embedding dimension $D=100$, unless noted otherwise. In the experiments that use \texttt{InfEmbed-Rule}, we used branching factor B=3. The rationale is that B should not be too large, to avoid unnecessarily dividing large slices with sufficient low accuracy into smaller slices. In practice, B=2 and B=3 did not give qualitatively different results.

\subsection{\texttt{InfEmbed} on \texttt{dcbench}}
\label{dcbench}
\textbf{Overview \& Experimental Procedure}: \texttt{dcbench} \citep{eyuboglu2022dcbench} provides 1235 pre-trained models that are derived from real-world data, where the training data is manipulated in 3 ways to lead to test-time errors. This gives 3 tasks - discovering slices whose errors are due to the following manipulations: 1) ``rare'' (examples in a given class are down-sampled), 2) ``correlation'' (a spurious correlation is introduced), and 3) ``noisy label'' (examples in a given class have noisy labels). Since the error causing slice is known a priori, \texttt{dcbench} can serve as a way to assess a new SDM. The underlying datasets for those tasks include 3 input types: natural images, medical images, and medical time-series.
~\citet{eyuboglu2022dcbench} also introduced Domino, the SDM that is currently the best-performing SDM on \texttt{dcbench}, which uses multi-modal CLIP embeddings as input to an error-aware mixture model. Following~\citet{eyuboglu2022dcbench}, we compare \texttt{InfEmbed} to Domino using the precision-at-k measure, which measures the proportion of the top k ($k=10$) elements in the discovered slice that are in the ground truth slice. We evaluate in the setting where errors are due to a trained model.

\begin{wraptable}{r}{7.0cm}
\centering
\resizebox{.95\linewidth}{!}{
\begin{tabular}{cccc}
\hline
 & \textbf{Rare} & \textbf{Correlation} & \textbf{Noisy Label}\\
\textbf{SDM Approach}   &  &  & \\
\hline
\multicolumn{4}{c}{\textbf{Natural Images}} \\
\hline
\hline
 Domino   & 0.4 &  0.45  & 0.6 \\
Inf-Embed &  0.65 &  0.55 & 0.67   \\
\hline
\multicolumn{4}{c}{\textbf{Medical Images}} \\
\hline
Domino    &  0.39  & 0.6 & 0.58\\
Inf-Embed & 0.57 &  0.62 & 0.73   \\
\hline
\multicolumn{4}{c}{\textbf{Medical Time Series}} \\
\hline
Domino    & 0.6   & 0.55 & 0.9\\
Inf-Embed &  0.64 & 0.65  & 0.81   \\
\hline
\end{tabular}
}
\caption{\texttt{InfEmbed} generally outperforms Domino across 3 input types and 3 tasks on the \texttt{dcbench} benchmark in terms of precision (higher is better).}
\label{tab:dominotrainedmodel}
\vspace{-15pt}
\end{wraptable}

\textbf{Results}: Table~\ref{tab:dominotrainedmodel} compares \texttt{InfEmbed} to Domino across the 3 tasks and 3 input types. \texttt{InfEmbed} always beats Domino except on the ``noisy'' task for the EEG Medical Time series input type.  We can also compute \textbf{coherence scores}---the k-means objective within each cluster---for each setting in the domino benchmark, 

We also compute the coherence scores, the K-means objective that we cluster, across all slices for influence embedding clustering as well as for domino. We find a similar correspondence to the previous results. In every setting where we outperform Domino (1219 out of 1235 trained settings), our coherence scores are better, in some cases, by almost 50 percent. Regarding label homogeneity: In high performing clusters, we find that influence embedding clustering has lower label homogeneity than domino. However, for error clusters we find the reverse. The goal of our scheme is to find clusters where the model is ‘wrong for the same reason’. We conjecture that influence embeddings are most effective for these settings.

\subsection{\texttt{InfEmbed} on the SpotCheck Benchmark}
\label{spotcheck}
\begin{wraptable}{r}{7.5cm}
\centering
\rowcolors{1}{}{lightgray}
\resizebox{.95\linewidth}{!}{
\begin{tabular}{ccc}
\hline
 & \textbf{DR} & \textbf{FDR} \\
\textbf{SDM Approach}   &  & \\
\hline
Barlow   & 0.43(0.04) &  0.03(0.01) \\
Spotlight &  0.79(0.03) &  0.09 (0.55)  \\
Domino    &  0.64 (0.04)  & 0.07(0.01)\\ 
PlaneSpot & 0.85 (0.03) & 0.07 (0.01) \\
Inf-Embed & 0.91 (0.09) &  0.15 (0.62) \\
\hline
\end{tabular}
}
\caption{Spotcheck Benchmark results.}
\label{tab:spotcheck}
\end{wraptable}

\textbf{Overview \& Experimental Procedure.} The SpotCheck benchmark~\cite{plumb2022evaluating} is based on a synthetic task consisting of 3 semantic features that can be easily controlled to determine the number of `blindspots' in a dataset. A blindspot is a feature responsible for a model's mistake. Similar to  \texttt{dcbench}, a model trained on data generated from SpotCheck is induced to make mistakes on an input that has a set of blindspots. The task is to predict the presence of a square in the image. Attributes of the input image such as the background, object color, and co-occurrence can be varied to induce mistakes in models obtained from the manipulated data generation process. ~\citet{plumb2022evaluating} find that current SDM approaches struggle in the presence of several blindspots. Consequently, they propose PlaneSpot, an SDM that uses the representation of the model's last layer, projected to 2 dimensions, along with the model's `confidence' as part of a mixture model to partition the dataset. They show that PlaneSpot outperforms current approaches as measured by discovery and false discovery rates. We replicate the SpotCheck benchmark, for a subset of features, and compare $\operatorname{InfEmbed}$ to PlaneSpot. We refer to the Appendix Section~\ref{appendix:sc:datasets} for additional details.

\textbf{Results.} We report the performance of the influence embedding approach compared to PlaneSpot on the Spotcheck benchmark in Table~\ref{tab:spotcheck}. We find that Inf-Embed approaches improves upon PlaneSpot's performance on the benchmark, which indicates that the proposed approach is an effective SDM.

\textbf{Alternative Clustering Approaches.} 
To check that the proposed scheme is robust to the choice of clustering algorithm, we replicate the spotcheck experiments with 2 additional clustering algorithms: 1) DBSCAN, and 2) Spectral Clustering. 
Similar to Table~\ref{tab:spotcheck}, we find that the discovery rate for these algorithms are: 0.92, and 0.89 respectively.
In both both cases, the performance of the $\operatorname{InfEmbed}$ procedure does not degrade, which indicates that the influence embedding representations are robust to the clustering algorithm used.

\subsection{Imagenet \& High Dimensional Settings}
\begin{figure*}[t]
\centering
\includegraphics[page=1, scale=0.20]{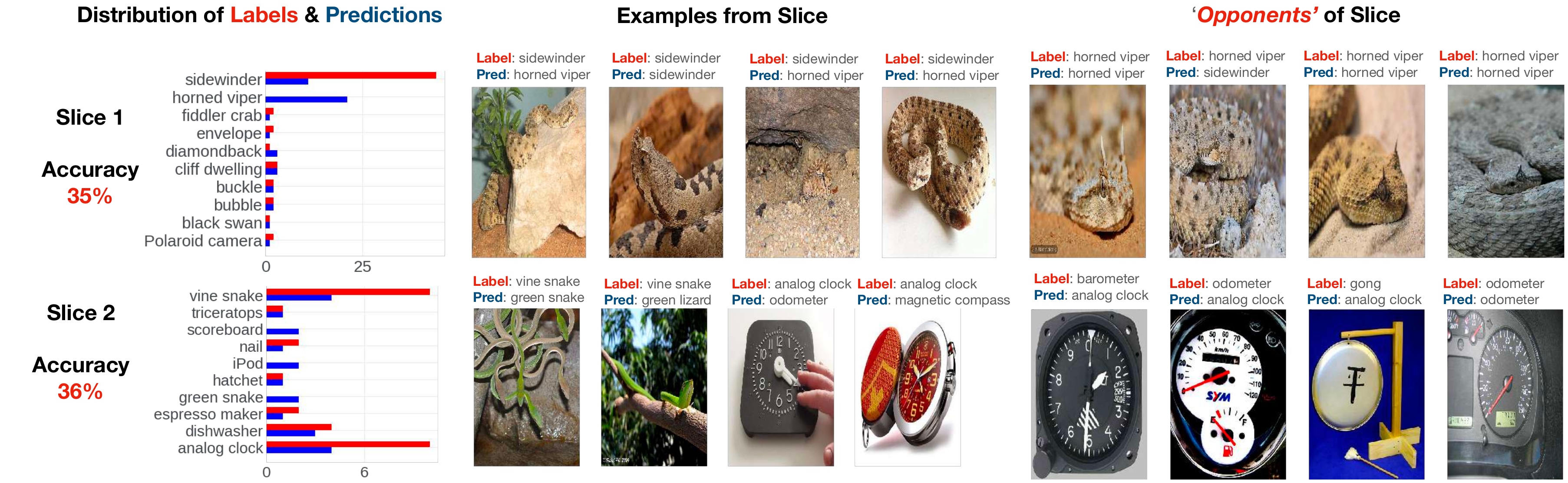}
\caption{\textbf{Overview of Influence embedding Clustering for a Resnet-18 on ImageNet.} We show two slices from applying the influence embedding approach to the ImageNet data. We apply the $\operatorname{RuleFind}$ algorithm and identify slices with accuracy under $40$ percent and atleast 25 samples. \textbf{Left}: On the left, we show the distribution of labels and model predictions in that slice (Labels are in red, and Prediction are blue). \textbf{Middle}: Representative examples from the slice. \textbf{Right}: We show four opponent examples for the identified slice.}
\vspace{-11pt}
\label{fig:imagenet}
\end{figure*}

\input{sections/imagenet.tex}

\input{sections/agnews.tex}

\subsection{Detecting Spurious Signals in Boneage Classification}\label{sec:boneage}
\textbf{Overview \& Experimental Procedure}: We now artificially inject a signal that spuriously correlates with the label in training data, and see if \texttt{InfEmbed-rule} can detect the resulting test errors as well as the spurious correlation. Here, the classification problem is bone-age classification \footnote{\url{https://www.kaggle.com/datasets/kmader/rsna-bone-age}} - to predict which of 5 age groups a person is in given their x-ray.  Following the experimental protocol of ~\citet{zhou2022feature} and ~\citet{adebayo2021post}, we consider 3 possible spurious signals: tag, stripe, or blur (see Appendix for additional details). For a given signal, we manipulate the bone-age training dataset by injecting the spurious signal into all \emph{training} examples from the ``mid-pub'' class (short for mid-puberty).  Importantly, we leave the test dataset untouched, so that we expect a model trained on the manipulated training data to err on ``mid-pub'' test examples.

This setup mimics a realistic scenario where x-ray training data comes from a hospital where one class is over-represented, and also contains a spurious signal (i.e. x-rays from the hospital might have the hospital's tag).  We then train a resnet50 on the manipulated training data, and apply \texttt{InfEmbed-Rule} on the untouched test data to discovered slices with at least 25 examples and at most 40\% accuracy (same as for the Imagenet analysis).  We repeat this manipulate-train-slice discovery procedure for 3 runs: once for each possible signal.

\textbf{Results}:  For all 3 runs / signals, \texttt{InfEmbed-Rule} returned 2 slices satisfying the rule.  From the distribution of labels and predictions, we observe that for each signal, the labels in the most problematic slice are of the ``mid-pub'' class, showing that \texttt{InfEmbed-Rule} is able to detect test errors due to the injected spurious training signal.  Furthermore, for each run, around 80\% of the slice's opponents are training examples with the spurious training signal, i.e. examples with the ``mid-pub'' label.  Together, this shows \texttt{InfEmbed-Rule} can not only detect slices that the model errs on because of a spurious training signal, but also identify training examples with the spurious training signal by using slice opponents.

%% file: sections/imagenet.tex
\textbf{Overview \& Experimental Procedure:} Here, we identify problematic slices in a ``natural'' dataset - the test split of Imagenet \citep{deng2009imagenet}, containing 5000 examples. This image-classification problem is a 1000-class problem, and the analyzed model is a pre-trained Resnet-18 model achieving 69\% test accuracy. To compute influence embeddings, we consider gradients in the ``fc'' and ``layer4'' layers, which contain ~8.9M parameters.  We use \texttt{ImfEmbed-Rule} to find slices with at most 40\% accuracy, and at least 25 examples. To diagnose the root-cause of errors for a slice, we examine the slice's opponents, following Section \ref{sec:method:proponents}.

\textbf{Results:}  We find 25 slices, comprising a total of 954 examples, whose overall accuracy is 34.9\%.  Figure~\ref{fig:imagenet} displays 2 of these slices (See Section~\ref{appendix:ix} for additional examples). For each slice, we show the distribution of predicted and true labels, the 4 examples nearest the slice center in influence embedding space, and the 4 strongest opponents of the slice. First, we see that the slices have the label homogeneity property as explained in Section \ref{sec:properties}---the predicted and true labels typically come from a small number of classes.  In the first slice, the true label is mostly ``sidewinder``, and often predicted to be ``horned viper''.  The slice's strongest opponents are horned vipers, hard examples of another class, which the model considers to be similar to sidewinders. In the second slice, we observe \emph{two} dominant labels. Here, vine snakes and analog clocks are often mis-predicted.  A priori, one would not have guess the predictions, for these seemingly, unrelated classes would be wrong for the same reasons.  However, looking at the opponents, we see all of them have a clock hand, which turns out to be similar to a snake.  Also, because the opponents are of a variety of classes, the mis-predictions are also of a variety of classes. Overall, we demonstrate here that our procedure can identify problematic slices with low performance and then trace the cause of the errors to a subset of the training data.

%% file: sections/agnews.tex
\subsection{AGNews \& Mis-labelled Data}

\textbf{Overview \& Experimental Procedure}:  We identify problematic slices in another ``natural'' dataset - the test split of AGNews \citep{NIPS2015_250cf8b5}, comprising 7600 examples.  This text-classification problem is a 4-class problem, and the pre-trained model is a BERT base model fine-tuned on the training dataset \footnote{\url{https://huggingface.co/fabriceyhc/bert-base-uncased-ag_news}}, achieving 93\% test accuracy, making 475 errors.  We use \texttt{ImfEmbed} with $K=25$, and examine slices with at most 10\% accuracy, and at least 10 examples.

\textbf{Results}: We find 9 such slices, comprising a total of 452 examples, whose overall accuracy is only 3\%.  Table \ref{tab:agnewst_short} shows 2 slices, both of which have label homogeniety.  In the first slice, the model achieves 0\% accuracy, systematically mis-predicting ``business'' examples to be ``world'', perhaps predicting any text with a country name in it to be ``world''.  In the second slice, most labels are ``science/technology'', and most predictions are ``business''.  However looking at the examples, they may in fact be mis-labeled, as they are about businesses and their technology.  This shows that \texttt{InfEmbed} can also identify \emph{mis-labeled data}.

\begin{table*}[]
\vspace{-10pt}
\centering
\resizebox{\textwidth}{!}{\begin{tabular}{p{1.0\linewidth}} 
\hline
{Slice with 45 examples, 0\% accuracy.  Predictions: 100\% ``world'', labels: 100\% ``business''} \\
\hline
1. \emph{``British grocer Tesco sees group sales rise 12.0-percent. Britain's biggest supermarket chain ...''}\\
2. \emph{``Nigerian Senate approves \$1.5 bln claim on Shell LAGOS - Nigeria's Senate has passed ... ''}\\
3. \emph{``Cocoa farmers issue strike threat. Unions are threatening a general strike in the Ivory Coast ...''}\\
\\
\hline
{Slice with 139 examples, 3\% accuracy.  Predictions: 93\% ``business'', labels: 94\% ``sci/tech''}\\

\hline
1.\emph{``Google Unveils Desktop Search, Takes on Microsoft. Google Inc. rolled out a version of its ...''} \\
2.\emph{``PalmOne to play with Windows Mobile? Rumors of Treo's using a Microsoft operating system ...''}\\
3.\emph{``Intel Posts Higher Profit, Sales. Computer-chip maker Intel Corp. said yesterday that earnings ...''}
\end{tabular}
}
\caption{Representative examples for 2 slices obtained using \texttt{InfEmbed} on AGnews.  In the 1st slice, the model mis-predicts business text containing country names to be ``world''.  The 2nd slice appears to be mis-labelled, showing \texttt{InfEmbed} can detect mis-labeled data.}
\label{tab:agnewst_short}
\vspace*{-10pt}
\end{table*}

%% file: sections/relatedwork.tex
\section{Related Work}
Error analysis, more generally, has been studied extensively even within the deep learning literature; however, here we focus our discussions on the recently formalized \textit{slice discovery problem}---as termed by \citet{eyuboglu2022domino}. 

\textbf{Error Analysis, Slice Discovery, Related Problems}:
To the best of our knowledge, \texttt{InfEmbed} is the first SDM based on influence functions.  More broadly, it represents the first use of influence functions for clustering and global explainability; they have been used for local explainability, i.e. returning the influence explanation for a \emph{single} test example, but not for clustering, as we do. \citet{pruthi2020estimating} also define influence embeddings, but use a different definition and also use it for a different purpose - accelerating the retrieval of an example's most influential training examples. Crucially, their embeddings were high-dimensional (exceeding the number of model parameters) and thus ill-suited for clustering, and lacked any dimension reduction procedure preserving the dot-product property, aside from random projections.

Previous SDMs use a variety of representations as input into a partitioning procedure.  Some rely on \emph{external} feature extractors, using PCA applied to pre-trained CLIP embeddings \citep{eyuboglu2022domino} along with an ``error-aware'' mixture model, or features from a robustly pre-trained image model \citep{singla2021understanding} along with a decision tree. This external reliance limits the generalizability of these SDMs---what if one wanted to perform slice discovery on a domain very different from the one CLIP was trained on, or wanted to consider a new modality, like audio or graphs? Other SDM's use the pre-trained model's last-layer representations \citep{d2022spotlight}, or transform them via SCVIS \citep{plumb2022evaluating} or UMAP \citep{sohoni2020no} before applying K-Means clustering.  While they are performant, it is less clear \emph{why} this is. In contrast, $\operatorname{InfEmbed}$ does not rely on external feature extractors and is theoretically framed by our influence function-based definition of coherency.

Other methods solve problems related to, but not the same as slice discovery. \citet{rajani2022seal, wiles2022discovering, hua2022discover} leverages the aforementioned SDM's to create \emph{interactive} systems for discovering slices, focusing on improving the user interface and how to explain slices.  In contrast, we focus on solving the core slice discovery problem. All these methods use last-layer representations along with K-Means or an ``error-aware'' mixture model \citep{hua2022discover}. \citet{jain2022distilling} globally \emph{ranks} examples to identify examples in a \emph{single} failure mode, but being unable to output multiple slices, is not a SDM. Numerous works also rank examples using scores capturing various properties: the degree to which a sample is out-of-distribution \citep{liu2020energy}, reliability \citep{schulam2019can}, or difficulty \cite{simsek2022understanding}.  Complementary to SDMs are methods for identifying low-performance subgroups when tabular metadata is available \citep{lakkaraju2017identifying}.
 
\textbf{Influence Functions}: Influence functions \citep{koh2017understanding} have been used to calculate influences, which can be used for ranking training examples for local explainability \citep{barshan2020relatif} for both supervised and unsupervised models \citep{kong2021understanding}, identifying mis-labeled data interactively \citep{teso2021interactive}, data augmentation \citep{lee2020learning}, active learning \citep{liu2021influence}, quantifying reliability \citep{schulam2019can}, and identifying mis-labeled data \citep{pruthi2020estimating}. However, influence functions have not been used for slice discovery or global explainability.

%% file: sections/conclusion.tex
\section{Conclusion}
We present a method, \texttt{InfEmbed}, that addresses the slice discovery problem. Our proposed solution departs from previous work in that it both identifies slices---group of data points--- on which a model under-performs, and that satisfy a certain \textit{coherence} property. We formalize coherence---predictions being wrong for the same reasons within a slice---as a key property that all slice discovery methods should satisfy. In particular, we leverage influence functions \citep{koh2017understanding} to define coherent slices as ones whose examples are influenced similarly by the training data, and then develop a procedure to return coherent slices based on clustering a representation we call influence embeddings. 
We demonstrate on several benchmarks that \texttt{InfEmbed} out-performs current approaches on a variety of data modalities. In addition, we show through case studies that \texttt{InfEmbed} is able to recover known errors in a variety of settings, and identifies coherent slices on realistic models.

\textbf{Limitations.} One major limitation of the proposed approach is that the validation set needs to include the kind of error that one is seeking to detect. In addition, the proposed procedure suggests to inspect the slice opponents as the root cause of the wrong predictions in a slice. However, it is still up to the user of the proposed Algorithm to infer which feature(s) of the slice opponents is responsible for causing the errors. Despite these limitations, error discovery by clustering influence embeddings can be seen as additional tool in the ML model debugging toolbox.

%% file: sections/appendix.tex
\newpage 
\onecolumn
\part{Appendix} 
\parttoc 

\section{Additional Related Work}
\label{appendix:relatedwork}
\textbf{Interpretability and Post-hoc Explanations.} Despite initial evidence that explanations might be useful for detecting that a model is reliant on spurious signals~\citep{lapuschkin2019unmasking, rieger2019interpretations}, a different line of work directly counters this evidence.~\citet{zimmermann2021well} showed that feature visualizations~\citep{olah2017feature} are not more effective than dataset examples at improving a human's understanding of the features that highly activate a DNN's intermediate neuron. Increasing evidence demonstrates that current post hoc explanation approaches might be ineffective for model debugging in practice~\citep{chen2021towards, alqaraawi2020evaluating, ghassemi2021false, balagopalan2022road, poursabzi2018manipulating, bolukbasi2021interpretability}. In a promising demonstration,~\citet{lapuschkin2019unmasking} apply a clustering procedure to the LRP saliency masks derived from a trained model. In the application, the clusters that emerge are able to separate groups of inputs where, presumably, the model relies on different features for its output decision. This work differs from that in a key way: ~\citet{lapuschkin2019unmasking} demonstration is to seek understanding of the model behavior and not to perform slice discovery. There is no reason why a low performing cluster should emerge from such clustering procedure.

\section{Practical Low-Rank Influence Function}\label{appendix:lowinf}
 The main challenge in computing influence is that it is impractical to explicitly form the $H_{\theta}$ needed to compute influence in 
 $I(\tr{z},\te{z}) \coloneqq \nabla_{\theta} L(\tr{z};\theta)^{\intercal} H_{\theta}^{-1} \nabla_{\theta} L(\te{z};\theta),$
 unless the model is small, or if one only considers parameters in a few layers. \citet{schioppa2022scaling} address this problem by forming a low-rank approximation of $H_{\theta}^{-1}$ without explicitly forming $H_{\theta}$. Their low-rank implementation first applied the Arnoldi iteration \cite{trefethen1997numerical} to compute an orthonormal basis $(q_1,\mydots,q_D)$ for the $P$-dimensional Arnoldi subspace $(b,H_{\theta}b,\mydots,H_{\theta}^{P-1})$, as well as the restriction $R$ of $H_\theta$ to the subspace, so that $H_{\theta}=QRQ^{\intercal}$, where $Q\coloneqq [q_1,\mydots,q_P] \in \mathbb{R}^{|\theta|\times P}$ and $R \in \mathbb{R}^{P\times P}$. They choose $P$ to be 200 in their experiments.  Crucially, it is computationally feasible to run the Arnoldi iteration even on large models, because it only requires access to $H_\theta$ through Hessian-vector products, not through its explicit formation.  Then, leveraging the fact that the Arnoldi subspace of a matrix tends to contain its top eigenvectors, they approximate $R\approx V\lambda V^{\intercal}$ and $R^{-1}\approx V\lambda^{-1} V^{\intercal}$, where $V=[v_1,\mydots,v_D]\in\mathbb{R}^{P \times D}$, $v_1,\mydots,v_D$ are the top $D$ eigenvectors of $R$, and $\lambda=\operatorname{diag}(\lambda_1,\mydots,\lambda_D)$, where $\lambda_1,\mydots,\lambda_D$ are the corresponding eigenvalues.  They choose $D$ to be around 50 in their experiments.  Crucially, it is computationally feasible to find the top eigenvectors and eigenvalues via eigen-decomposition on an explicitly-formed $R$, because $R$ is of size $P \times P$ with $P=200$.  Finally, they form a low-rank approximation of $H_{\theta}^{-1}$, 
\begin{align}
\hat{H}_{\theta}^{-1} &\coloneqq M\lambda^{-1}M^{\intercal},\ \text{where}\ M=QV\label{eq:H_hat_appendix}
\intertext{and use it to define the \emph{practical influence} of training example $\tr{z}$ on test example $\te{z}$:}
\hat{I}(\tr{z},\te{z}) &\coloneqq \nabla_{\theta} L(\tr{z};\theta)^{\intercal} \hat{H}_{\theta}^{-1} \nabla_{\theta} L(\te{z};\theta).\label{eq:influence_hat_appendix}
\end{align}
Algorithm \ref{alg:arnoldi_factor} outlines finding the factors needed for $\hat{H}_{\theta}^{-1}$.  Note that for brevity of notation, throughout we will let \emph{configuration} $\Theta\coloneqq(L,f,\theta)$ denote the loss function, model, and parameters.
\algnewcommand{\LeftComment}[1]{\Statex \(\triangleright\) #1}
\begin{algorithm}
\label{alg:arnoldi_factor}
\caption{Finding low-rank factors of Hessian}
\begin{algorithmic}
\Procedure{FactorHessian}{$\tr{\mybold{Z}}$, $\Theta$, $P$, $D$}
\State\textbf{Inputs:} training data $\tr{\mybold{Z}}$, configuration $\Theta$, Arnoldi dimension $P$, rank $D$
\State $H_\theta \gets \sum_{i=1}^{N'} \nabla_{\theta}^2 L(\idx{\tr{z}}{i};\theta)$ \Comment{implicitly define HVP}
\State Run Arnoldi iteration on $H_\theta$ for $P$ iterations to get $Q\in \mathbb{R}^{|\theta| \times P}, R\in \mathbb{R}^{P\times P}$
\State $V, \lambda \gets$ top-$D$ eigenvectors / values of $R$ via SVD
\State $M \gets QV$
\State \textbf{Return:} $M, \lambda$
\EndProcedure
\end{algorithmic}
\end{algorithm}

\subsection{Proof of Lemma 1}
\label{appendix:propoppo}
Influence embeddings satisfy a critical property - that if two examples have similar influence embeddings, they also tend to have similar influence explanations.  This is formalized by the following lemma:
\begin{lemma}
There exists a constant $C$ such that for any test examples $\idx{\te{z}}{i},\idx{\te{z}}{j}$, $||E(\idx{\te{z}}{i}) - E(\idx{\te{z}}{j})||^2 \leq C||\mu(\idx{\te{z}}{i}) - \mu(\idx{\te{z}}{j}))||^2$.
\end{lemma}
\begin{proof}
The proof follows from the Cauchy-Schwartz inequality, and the fact that influence is the dot-product of influence embeddings.
\begin{align*}
||E(\mbf{\idx{\te{z}}{i}}) - E(\mbf{\idx{\te{z}}{j}})||^2 &= \textstyle \sum_{n=1}^{N'} (\hat{I}(\mbf{\idx{\tr{z}}{n}},\idx{\te{z}}{i}) - \hat{I}(\idx{\tr{z}}{n},\idx{\te{z}}{j}))^2\\
&=\textstyle\sum_{n=1}^{N'} (\mu(\idx{\tr{z}}{n})^{\intercal}\mu(\idx{\te{z}}{i}) - \mu(\idx{\tr{z}}{n})^{\intercal}\mu(\idx{\te{z}}{j}))^2\\ 
&=\textstyle\sum_{n=1}^{N'} (\mu(\idx{\tr{z}}{n})^{\intercal}(\mu(\idx{\te{z}}{i}) - \mu(\idx{\te{z}}{j})))^2\\
&\leq\textstyle\sum_{n=1}^{N'} ||\mu(\idx{\tr{z}}{n})||^2||\mu(\idx{\te{z}}{i}) - \mu(\idx{\te{z}}{j})||^2\\
&\leq C ||\mu(\idx{\te{z}}{i}) - \mu(\idx{\te{z}}{j})||^2,\ \text{where}
\end{align*}
$C\coloneqq \sum_{n=1}^{N'} ||\mu(\idx{\tr{z}}{n})||^2$ does not depend on $i$ nor $j$.
\end{proof}



\subsection{Description of \texttt{InfEmbed-Rule}}\label{sec:infembedrule}
Note that the key hyperparameter of the $\operatorname{InfEmbed}$ method is $K$, the number of slices to return.  In practice, it may not be intuitive for a user to choose $K$.  Instead, the user may want to know if there exists any coherent slices that are problematic, as defined by satisfying a rule: has accuracy less than some threshold, and number of examples above some threshold.  Therefore, we also propose a procedure that recursively clusters influence embeddings until slices satisfying the rule are found, or until the slices are too small.  The approach is analogous to building a tree to identify slices satisfying the rule, where the splits are determined by K-Means clustering of influence embeddings.  In addition to letting the user specify more intuitive hyperparameters, this procedure also has the advantage that if a large slice with sufficiently low accuracy is found, it will not be clustered further.  Algorithm \ref{alg:basic} outlines this \texttt{InfEmbed-Rule} method.  Its inputs are the same as \texttt{InfEmbed}, except instead of specifying $K$, the number of slices, one specifies accuracy threshold $A$, size threshold $S$, and branching factor $B$, which specifies how many clusters the K-means call in each step of the recursion should return.  Its outputs is a set of slices, each having accuracy less than $A$ and size greater than $S$.  In practice, the recursion proceeds to a maximum depth.

\begin{algorithm}\label{alg:rulefind}
\label{alg:rule}
\caption{Procedure finding slices with low accuracy and large size}
\begin{algorithmic}
\Procedure{RuleFind}{$\boldsymbol{\mu}, A, S, B$}
\State\textbf{Inputs:} $\boldsymbol{\mu}$ a list of influence embeddings, accuracy threshold $A$, size threshold $S$, branching factor $B$
\State\textbf{Outputs:} a set of lists of influence embeddings.  Each set of influence embeddings corresponds to a slice with accuracy < $A$ and size > $S$
\State $acc \gets $ accuracy of examples represented in $\{\mu_i\}$
\State $size \gets $ number of examples represented in $\{\mu_i\}$
\If{$acc \leq A$ and $size \geq S$}
\State \textbf{Return}: $\{\boldsymbol{\mu}\}$
\EndIf
\If{$size < S$}
\State \textbf{Return:} $\{\}$
\EndIf
\State $\boldsymbol{r} \gets \operatorname{K-Means}(\boldsymbol{\mu}, B)$ \Comment{get cluster assignments}
\State $\boldsymbol{r}_k \gets \{\mu_i\}_{i:r_i=k}$ for $k \in [B]$ \Comment{embeddings for each cluster}
\State $F=\{\}$
\For{$k\in[B]$} \Comment{search within each cluster}
\State $F \gets F \cup\{\operatorname{RuleFind}(\{\mu_i\}_{i:r_i=b}, A, S, B)\}$
\EndFor
\State \textbf{Return:} $F$
\EndProcedure
\Procedure{InfEmbed-Rule}{$A, S, B, \te{\mybold{Z}},\tr{\mybold{Z}},\Theta,P, D$}
\State $\boldsymbol{\mu}\gets \operatorname{GetEmbeddings}(\te{\mybold{Z}},\tr{\mybold{Z}},\Theta,P, D)$
\State $F \gets \texttt{RuleFind}(\boldsymbol{\mu}, A, S, B)$ \Comment{a set of lists of embeddings, with each embedding corresponding to a test example}
\State \textbf{Return}: the partition of test dataset induced by $F$
\EndProcedure
\end{algorithmic}
\end{algorithm}

\section{DcBench}
\label{appendix:dc}

\textbf{Overview.} ~\citet{eyuboglu2022domino} formalized the slice discovery problem and introduced Dcbench~\citep{eyuboglu2022dcbench}, a benchmark, consisting of pre-trained models across a variety of data sets for testing any new SDM. For each dataset in the benchmark, a manipulation is applied to the dataset in order to induced one of three kinds of errors in a collection of samples---slice---, and a model is trained on the modified dataset to exhibit the injected error. To assess a new SDM, the partitioning returned by the method is then compared to the ground-truth slices used to generate the model. The collection of datasets in dcbench includes natural images (ImageNet, and CelebA), Medical images based on the MIMIC Chest X-Ray, and Medical Time-Series Data of electroencephalography (EEG) signals. A collection of 1235 models were trained for various dataset, model, and error type groups. In addition to the benchmark, ~\citet{eyuboglu2022domino} introduced Domino, an SDM, that uses a mixture model that models the generative process of a slice. The input to the mixture model is a multi-model embedding (CLIP or ConVIRT) of the test data, and the mixture model assumes that the embedding, label, and predictions are independent conditioned on the slice. ~\citet{eyuboglu2022domino}  demonstrate that Domino outperforms or matches approaches based on last-layer or vision-transformer-based embeddings. Consequently, we compare against Domino on dcbench since it is, as of the writing of this paper, the state-of-the-art.

\textbf{Experimental Procedure.} The dcbench benchmark consists of a collection of slice discovery tasks. Each task includes a series of artifacts: the model, base dataset, model activations, validation and test predictions, clip activations for the test and validation data, and ground-truth slices. Domino was evaluated using two kinds of models: a synthetic and trained model. We restrict our focus to the trained models, only, since our proposed approach only applies to trained models. We compute the precision-at-10 metric, similar to domino, across all data modalities and error type. 

\textbf{Results.} We present a comparison of the Precision-at-10 metric between Domino and the influence embeddings approach in Table~\ref{tab:dominotrainedmodel}. The influence embedding approach outperforms domino across all settings except on the Noisy Label task for the EEG data modality.

\subsection{Datasets}
\label{appendix:dc:datasets}
All of the dcbench dataset, models, and task information is publicly available via the opensource repository: \url{https://github.com/data-centric-ai/dcbench}

\section{SpotCheck}
\label{appendix:sc:exp}
\textbf{Model}: Following~\citet{plumb2022evaluating}, we train a ResNet-18 model for each dataset and blindspot specification.

\subsection{Datasets}
\label{appendix:sc:datasets}
The Spotcheck benchmark is publicly described in the paper by ~\citet{plumb2022evaluating}. The SpotCheck
benchmark (Plumb et al., 2022) is based on a synthetic task
consisting of 3 semantic features that can be easily con-
trolled to determine the number of ‘blindspots’ in a dataset.
A blindspot is feature responsible for a model’s mistake.
Similar to dcbench, a model trained on data generated
from SpotCheck is induced to make mistakes on an input
that has a set of blindspots. The overall task is to predict
the presence of a square in the image. Attributes of the
input image such as the background, object color, and co-
occurrence can be varied to induce mistakes. Plumb et al.
(2022) find that current SDM approaches struggle in the
presence of several blindspots. Consequently, they pro-
pose PlaneSpot, an SDM that uses the representation of the
model’s last layer, projected to 2 dimensions, along with the
model’s ‘confidence’ as part of a mixture model to partition
the dataset. They show that PlaneSpot outperforms current
approaches as measured by discovery and false discovery
rates. We replicate the SpotCheck benchmark (for 4 variable features) and compare
InfEmbed to PlaneSpot.

\section{Imagenet \& High Dimensional Settings}
\label{appendix:ix}

\textbf{Overview \& Experimental Procedure:} The Imagenet validation data \cite{deng2009imagenet} contains 1000 classes, with 50 examples per class.  We use the $\texttt{InfEmbed-Rule}$ procedure to find slices with at most 40\% accuracy, and at least 25 examples.

We use a Resnet18 trained on Imagenet, provided in the torchvision PyTorch library.  The model achieves 69.8\% accuracy.  To compute influence embeddings, we consider gradients in the ``fc'' and ``layer4'' layers, which contain a total of ~8.9M parameters.  As in all experiments, we compute influence embeddings of $D=100$ dimensions using an Arnoldi dimension of $P=500$.  To explain the root-cause of predictions in a given slice, we compute the strongest opponents of the slice.  These are the training examples whose influence on the \emph{slice}, i.e. influence on the \emph{total} loss over all examples in the slice is the most negative, i.e. most harmful.  Due to the properties of influence embeddings, the influence of a training example on a slice is simply the dot-product between the influence embedding of the training example, and the \emph{sum} of the influence embeddings in the slice.  For simplicity, we search for opponents within the test dataset, i.e. treating the test data as the training data.  Since there is no distribution shift in Imagenet, for explanatory purposes this is an acceptable approximation.

\textbf{Results:}  Applying the above rule, we find 25 slices, comprising a total of 954 examples, whose overall accuracy is 34.9\%.  Figure \ref{fig:imagenetlong} displays 4 of the slices, where each row contains the distribution of predicted and true labels, the 4 examples nearest the slice center in influence embedding space, and the 4 strongest opponents of the slice. First, we see that the slices have the label homogeneity property as explained in Section \ref{sec:properties} - the predicted and true labels typically come from a small number of classes.  In the first slice (counting from the top), the true label is mostly ``sidewinder``, and often predicted to be ``horned viper''.  The slice's strongest opponents are horned vipers, hard examples of another class, which the model likely models similarly to sidewinders.  Thus the presence of the former drives the prediction of the latter towards horned vipers and away from sidewinders, increasing loss.  A similar story holds for the second slice, where breastplates are often predicted to be cuirasses.  Interestingly, the first opponent may actually be mis-labeled (cuirasses are breastplates fused with backplates, which it may lack), reflecting the fine line between hard and mis-labeled opponents examples.  The third and fourth slices are interesting in that each slice is predominantly of \emph{two} labels.  For example, in the third slice, border collies are strong opponents for both borzois and collies (which are a different dog breed than border collies), causing them to be mis-predicted as border collies.  This makes sense given that all three dog breeds look similar.  In the fourth slice, vine snakes and analog clocks are often mis-predicted.  A priori, we would not know that predictions for these seemingly unrelated classes would be wrong for the same reasons.  However, looking at the opponents, we see all of them have a clock hand, which turns out to be similar to a snake.  Also, because the opponents are of a variety of classes, the mis-predictions are also of a variety of classes.

\begin{figure}
\label{fig:imagenetlong}
\includegraphics[width=0.95\linewidth]{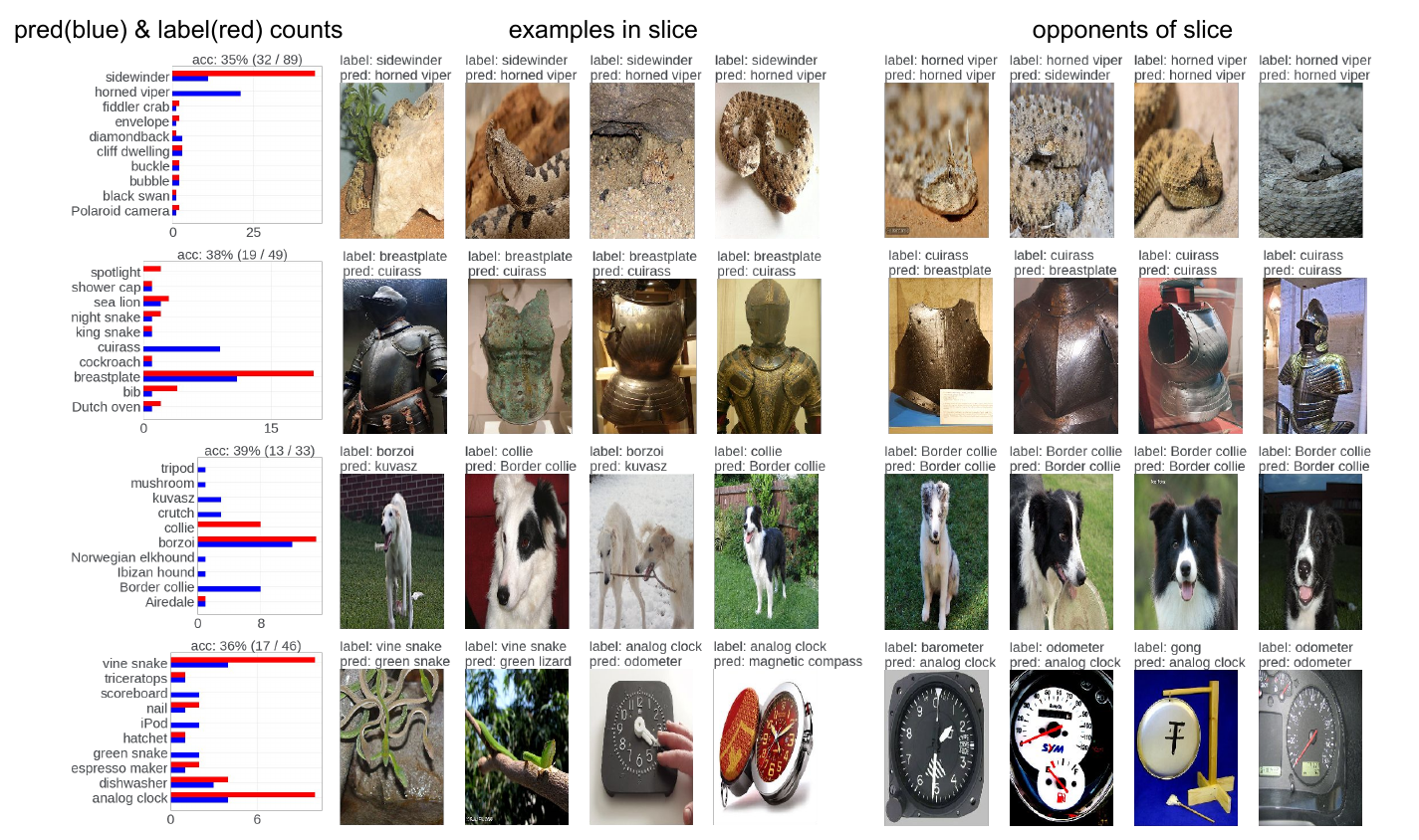}
\caption{For slices discovered by applying \texttt{InfEmbed-Rule} to Imagenet, we show the distributions of labels (red) and predictions (blue) in the slice (left), examples from the slice (middle), and opponents of the slice (right).}
\end{figure}

Finally, we also list a few other discovered slices: 1) a slice where ``gown'' (wedding dress) is mis-predicted to be ``groom'', whose opponents are images labeled ``groom'' containing both a groom and spuriously-correlated gown, 2) a slice where ``windsor tie'' is mis-predicted to be suit, due to similar spurious correlations, 3) a slice where mis-predicted ``sunglasses'' are due to examples which contain sunglasses, but are labeled as other present classes, like ``lipstick'' and ``bib'', 4) two different classes (i.e. index) actually refer to the same object - ``maillot``.

\section{AGNews \& Mis-labeled Data}
\label{appendix:agnews}

\textbf{Overview \& Experimental Procedure:} The AGnews test data \cite{NIPS2015_250cf8b5}
contains 4 classes (Business, Sci/Tech, Sports, World) with 1900 examples per class.  We use a BERT base model fine-tuned on the training set \footnote{https://huggingface.co/fabriceyhc/bert-base-uncased-ag\_news}, which achieves 93.75\% accuracy on the test data which results in 475 test errors.  We use the $\operatorname{InfEmbed}$ method of Algorithm \ref{alg:basic} with $K=25$ to find slices with at most 10\% accuracy, and at least 10 examples since the total number of errors is small.

To compute influence embeddings, we consider gradients in the ``bert.pooler.dense'' and ``classifier'' layers, which are the top 2 linear layers of the model and contain a total of 590K parameters parameters.  As in all experiments, we compute influence embeddings of $D=100$ dimensions using an Arnoldi dimension of $P=500$.  

\textbf{Results:}  We find 9 slices that account for 92\% of errors (438 out of 475).  Some interesting patterns we observe:
\begin{itemize}
    \item Sci/Tech examples predicted as Business (30\% ) \& Business predicted as Sci/Tech (25\%) are 55\% of errors.  
    \item Business examples predicted as World (10.3\%) \& World predicted as Business (9.4\%) are 19.7\% of errors.
    \item Sci/Tech examples predicted as World (8.5\%) \& World predicted as Sci/Tech (7.3\%) are 15.8\% of errors.
    \item Sports articles are least likely to get predicted for other genres and only a few Business and Sci/Tech examples get predicted as Sports ( 4.1\% and 2.7\% respectively of errors. )
\end{itemize}

Table \ref{tab:agnewst} displays representative examples (ie, those closet to the cluster centers ) from 3 of the slices.  For Sci/Tech news articles that are predicted as Business articles, we see Google, PalmOne and Intel all referenced, but with references to investment bank reports, Stock symbols and earning reports, the true class is a bit ambiguous and illustrates why the model has the most troubles distinguishing between the two ( they account for 55\% of errors ). In fact, it is possible examples in this slice are actually mis-labeled, showing that \texttt{InfEmbed} can be used to detect systematically mis-labeled data. For Business articles that are predicted to be World articles we see references to the grocer Tesco, Shell and cocoa farmers, but with reference to Britain, the Nigerian Senate and protests in the Ivory coast which again makes their class a bit ambiguous; accounting for ~20\% of model errors.  Finally, for Business examples that are predicted as Sports, which is one of the smaller error slices we find, we see less ambiguous error instances which the model has trouble with.

Furthermore, we display in Figure \ref{fig:agnews_hist} the distribution of labels and predictions for the 9 slices, as well as their error rate.  We see that the slices possess the label homogeniety property.

\begin{figure}
\includegraphics[scale=0.60]{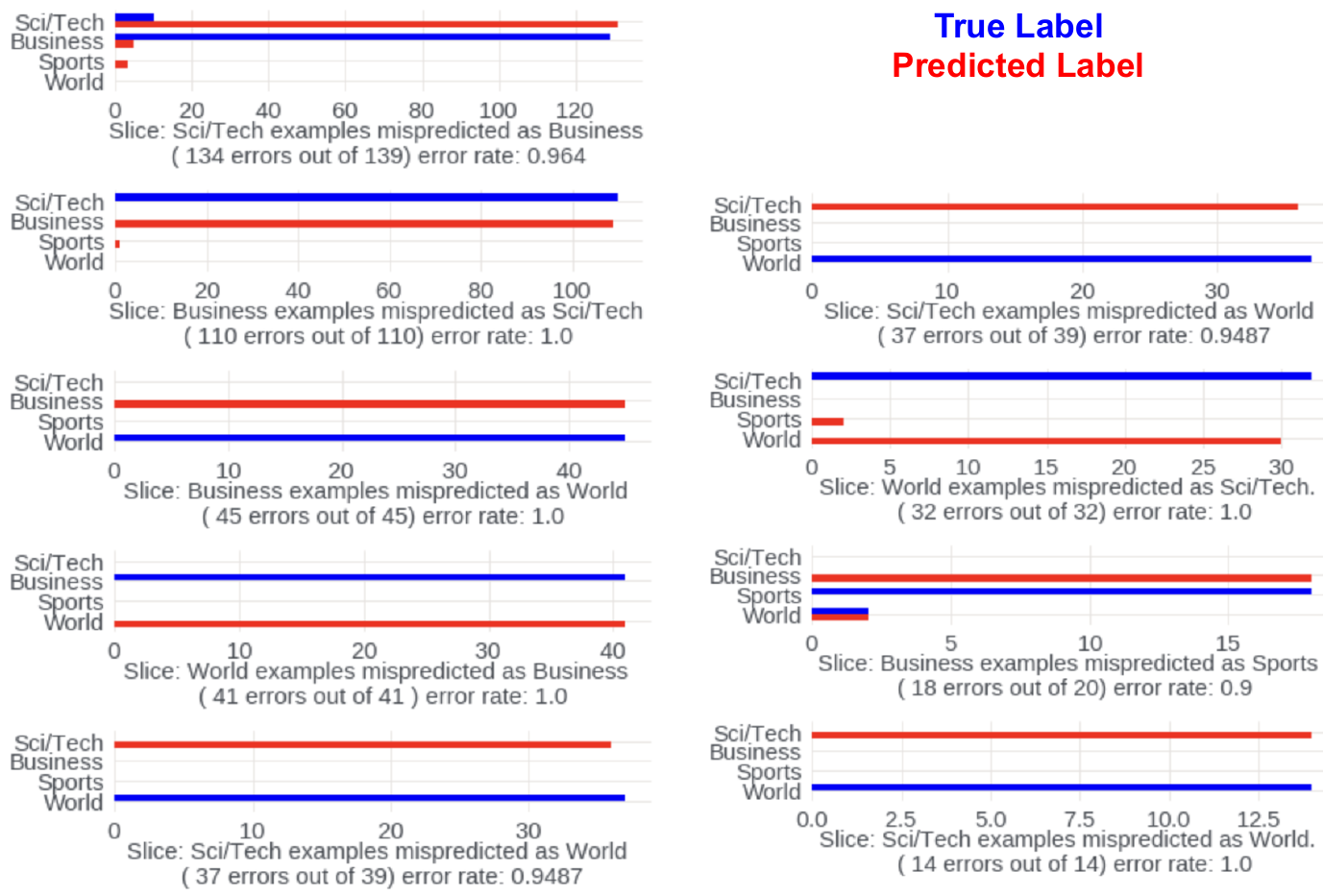}

\caption{Predicted vs True label distributions for 9 slices found for AGnews data}
\label{fig:agnews_hist}
\end{figure}

\begin{table}[t]
\centering
\begin{tabular}{p{0.95\linewidth}} 
{Examples from slice with 139 examples, 3\% accuracy.  Predictions are 93\% ``business'', labels are 94\% ``sci/tech''}\\
\hline
\hline
1. \emph{``Google Unveils Desktop Search, Takes on Microsoft Google Inc. (GOOG.O: Quote, Profile, Research) on Thursday rolled out a preliminary version of its new desktop search tool, making the first move against ...''} \\
\hline
2. \emph{``PalmOne to play with Windows Mobile? Rumors of Treo's using a Microsoft operating system have been circulating for more than three years. Now an investment bank reports that PalmOne will use a ...''}\\
\hline
3. \emph{``Intel Posts Higher Profit, Sales Computer-chip maker Intel Corp. said yesterday that earnings for its third quarter were \$1.9 billion -- up 15 percent from the same quarter a year ago ...''}
\\
\hline
\vspace{1pt}
{Examples from slice with 45 examples, 0\% accuracy.  Predictions are 100\% ``world'', labels are 100\% ``business''} \\
\hline
\hline
1. \emph{``British grocer Tesco sees group sales rise 12.0-percent (AFP) AFP - Tesco, Britain's biggest supermarket chain, said that group sales grew by 12.2 percent in the third quarter, driven by strong performances from its stores ...''}\\
\hline
2. \emph{``Nigerian Senate approves \$1.5 bln claim on Shell LAGOS - Nigeria's Senate has passed a resolution asking Shell's Nigerian unit to pay \$1.5 billion in compensation to oilfield communities for pollution, a Senate spokesman said. ''}\\
\hline
3. \emph{``Cocoa farmers issue strike threat. Unions are threatening a general strike in the Ivory Coast in a protest against the prices farmers are paid for their cocoa supplies.''}\\
\hline
\vspace{1pt}
{Examples from slice with 20 examples, 10\% accuracy.  Predictions are 90\% ``world'', labels are 90\% ``business''} \\
\hline
\hline
1. \emph{``Perry OKs money for APS as more accusations arise. The state's Adult Protective Services agency will get an emergency infusion of \$10 million to correct the kinds of problems that have arisen in El Paso.''} \\
\hline
2. \emph{Sign off, then sign in.  G. Michael Caggiano Jr. lies awake at night thinking about bank signs. He ponders them during breakfast, while brushing his teeth, and  quot;constantly quot; during the day, he says.''} \\
\hline
3. \emph{``Stanley set sights on Elland Road for casino Stanley Leisure plc has announced a Stanley Casinos Limited plan to develop a casino complex on land adjacent to Leeds United's Elland Road stadium.''} \\
\hline
\end{tabular}
\caption{Representative test examples for 3 high error slices obtained using \texttt{InfEmbed} on AGnews}
\label{tab:agnewst}
\end{table}

\section{Detecting Spurious Signals in Boneage Classification}
\label{appendix:rad}

\label{appendix:rad:datasets}

\textbf{Bone Age Dataset}: We consider the high stakes task of predicting the bone age category from a radiograph to one of five classes based on age: Infancy/Toddler, Pre-Puberty, Early/MiD Puberty, Late Puberty, and Post Puberty. This task is one that is routinely performed by radiologists and as been previously studied with a variety of DNN. The dataset we use is derived from the Pediatric Bone Age Machine learning challenge conducted by the radiological society of North America in 2017~\cite{halabi2019rsna}. The dataset consists of $12282$ training, $1425$ validation, and $200$ test samples. We resize all images to (299 by 299) grayscale images for model training. We note here that the training, validation, and test set splits correspond to similar splits used for the competition, so we retain this split.

 \textbf{Model and Hyperparameters}: We consider and a Resnet-50 model. The small DNN consists of: conv-relu-batchnorm-maxpooling operation successively, and two fully connected layers at the end. All convolutional kernels have stride $1$, and kernel size $5$. We train this model with SGD with momentum (set to $0.9$) and an initial learning rate of $0.01$. We use a learning rate scheduler that decays the learning rate every $10$ epochs by $\gamma=0.1$.

 \textbf{Overview \& Experimental Procedure}: We now artificially inject a signal that spuriously correlates with the label in training data, and see if \texttt{InfEmbed-rule} can detect the resulting test errors as well as the spurious correlation. Here, the classification problem is bone-age classification - to predict which of 5 age groups a person is in given their x-ray.  Following the experimental protocol of ~\citet{zhou2022feature} and ~\citet{adebayo2021post}, we consider 3 possible signals that can be added to an x-ray.  Importantly, spurious correlations involving these signal have been difficult to detect in past work.  For a given signal, we manipulate the bone-age training dataset by injecting the spurious signal into all \emph{training} examples from the ``mid-pub'' class (short for mid-puberty, one of the 5 age groups).  Importantly, we leave the test dataset untouched, so that we expect a model trained on the manipulated training data to err on ``mid-pub'' test examples; the spurious signal the model associated with ``mid-pub'' is missing, so that the model, having not associated other features with ``mid-pub'', finds no evidence to deem it ``mid-pub''.  This setup mimics a realistic scenario where x-ray training data comes from a hospital where one class is over-represented, and also contains a spurious signal (i.e. x-rays from the hospital might have the hospital's tag).  

We then train a resnet50 on the manipulated training data, and apply \texttt{InfEmbed-Rule} on the untouched test data to discovered slices with at least 25 examples and at most 40\% accuracy (same as for the Imagenet analysis).  To assess whether \texttt{InfEmbed-Rule} succeeded at detecting the model's reliance on the spurious training signal, we 1) take the most problematic slice, i.e. discovered slice with the lowest accuracy, and see if its labels are mostly ``mid-pub'' - the class we know a priori to be under-performing in the test dataset due to spurious training signal injection, and 2) examine whether the slice opponents are training examples with the spurious training signal, i.e. of the ``mid-pub'' class.  This lets us confirm whether the errors in the slice are due to the spurious training signal (as opposed to other root-causes), and whether using \texttt{InfEmbed-Rule} along with slice opponents analysis lets us discover the spurious training signal.  We repeat this manipulate-train-slice discovery procedure for 3 runs: once for each possible signal.

\textbf{Results}:  For all 3 runs / signals, \texttt{InfEmbed-Rule} returned 2 slices satisfying the rule.  In Figure \ref{fig:raddistproponents}, for each signal, we show the distribution of labels and predictions for the most problematic slice (left column) and the same distribution for that slice's top-50 opponents (right column).  We see that for each signal, the labels in the most problematic slice are of the ``mid-pub'' class, showing that \texttt{InfEmbed-Rule} is able to detect test errors due to the injected spurious training signal.  Furthermore, for each run, around 80\% of the slice's opponents are training examples with the spurious training signal, i.e. examples with the ``mid-pub'' label.  Together, this shows \texttt{InfEmbed-Rule} can not only detect slices that the model errs on because of a spurious training signal, but also identify the root-cause of those errors by using slice opponents to identify training examples with the spurious signal.

\section{Applying Influence Embeddings to Limited Data
Settings \& Using Attributes} \label{sec:covid}

We use the COVID-19 Chest X-Ray Dataset\footnote{ \url{https://github.com/ieee8023/covid-chestxray-dataset}}.
This dataset, provided in the torchxrayvision PyTorch library, has 535 images of chest X-rays. We use the COVID-19 label, a binary label, where there are 342 cases of COVID. 
We leverage a pretrained DenseNet model (densenet121-res224-all) provided in the torchxrayvision library, trained to predict multiple diseases, and adapt it to predict the COVID label by replacing the last fully-connected layer of the model.
We then fine-tune the model, including this last layer, on a training split of the COVID dataset using a binary cross-entropy loss. We follow standard normalization techniques for these datasets, including using a central crop of 224x224 pixels and normalizing the image values in the $[-1024, 1024]$ range.
The resulting model achieves an accuracy of 76.6\% on the test set of 107 points. We apply InfEmbed, taking gradients with respect to the last fully connected layer of the model. We then apply K-Means clustering with $k=3$. 

\textbf{Results.}
The first cluster had an accuracy of 85\%, capturing the majority of correctly classified samples. Besides correctly classified positive and negative samples, the first cluster also had five false positives and nine false negatives. The second cluster consisted of four samples that were false positives and two samples that were true negatives, while the third cluster had six samples that were false positives, and one sample that was a false negative.  
To a layperson not trained in interpreting radiology images, it may be hard to pick out differences between images in the three clusters. To study the differences between the three clusters, we inspected attributes and metadata provided in the dataset. 

Besides the COVID-19 label, the dataset had 18 additional labels for various diseases and infections, from Aspergillosis to Varicella. The full list can be found in Table \ref{tab:covid_labels}; see \cite{cohen2020covid} for more details. We computed the incidence of these diseases for each of the true negatives, true positives, false negatives, false positives samples by cluster. Table \ref{tab:covid_labels} presents the results. We observe that false positive samples in all three clusters tended to have pneumonia and tuberculosis, but not COVID-19. Looking at true positives, all the samples in cluster 1 that had COVID-19 also had pneumonia. However, samples with pneumonia do not always have COVID-19; rather, sometimes they have other lung diseases such as tuberculosis, SARS, etc. 

A differentiator between the three clusters is the incidence of additional diseases. Some false positive samples in cluster 2, in addition to pneumonia and tuberculosis, also had legionella, while those in cluster 3 in addition had SARS or pneumocystis. The false positive samples in cluster 1 had more diseases, such as Herpes, Klebsiella, and several others. Many of these diseases are lung diseases; these findings illustrate that the model may be having a harder time differentiating between different lung diseases and COVID-19.

\begin{table*}[ht!]
\centering
\resizebox{\textwidth}{!}{
\begin{tabular}{lrrr|rrr|rrr|rrr}
\toprule
{} & \multicolumn{3}{c}{\textbf{True Negatives}} & \multicolumn{3}{c}{\textbf{False Positives}} & \multicolumn{3}{c}{\textbf{False Negatives}} & \multicolumn{3}{c}{\textbf{True Positives}}\\
\cmidrule(lr){2-4}\cmidrule(lr){5-7}\cmidrule(lr){8-10}\cmidrule(lr){11-13}
{} &  C1  &  C2  &  \multicolumn{1}{c}{C3}  &  C1 &  C2 &  \multicolumn{1}{c}{C3} &  C1 &  C2 &  \multicolumn{1}{c}{C3} &  C1 &  C2 &  \multicolumn{1}{c}{C3} \\
\midrule
Aspergillosis  &   0.04 &    0.0 &    NaN &    0.0 &   0.00 &   0.00 &    0.0 &    NaN &    0.0 &    0.0 &    NaN &    NaN \\
COVID-19       &   0.00 &    0.0 &    NaN &    0.0 &   0.00 &   0.00 &    1.0 &    NaN &    1.0 &    1.0 &    NaN &    NaN \\
Chlamydophila  &   0.00 &    0.0 &    NaN &    0.0 &   0.00 &   0.00 &    0.0 &    NaN &    0.0 &    0.0 &    NaN &    NaN \\
H1N1           &   0.00 &    0.0 &    NaN &    0.0 &   0.00 &   0.00 &    0.0 &    NaN &    0.0 &    0.0 &    NaN &    NaN \\
Herpes         &   0.00 &    0.0 &    NaN &    0.2 &   0.00 &   0.00 &    0.0 &    NaN &    0.0 &    0.0 &    NaN &    NaN \\
Influenza      &   0.00 &    0.0 &    NaN &    0.0 &   0.00 &   0.00 &    0.0 &    NaN &    0.0 &    0.0 &    NaN &    NaN \\
Klebsiella     &   0.12 &    0.0 &    NaN &    0.2 &   0.00 &   0.00 &    0.0 &    NaN &    0.0 &    0.0 &    NaN &    NaN \\
Legionella     &   0.00 &    0.5 &    NaN &    0.0 &   0.25 &   0.00 &    0.0 &    NaN &    0.0 &    0.0 &    NaN &    NaN \\
MERS-CoV       &   0.00 &    0.0 &    NaN &    0.2 &   0.00 &   0.00 &    0.0 &    NaN &    0.0 &    0.0 &    NaN &    NaN \\
MRSA           &   0.00 &    0.0 &    NaN &    0.0 &   0.00 &   0.00 &    0.0 &    NaN &    0.0 &    0.0 &    NaN &    NaN \\
Mycoplasma     &   0.00 &    0.0 &    NaN &    0.0 &   0.00 &   0.00 &    0.0 &    NaN &    0.0 &    0.0 &    NaN &    NaN \\
Nocardia       &   0.04 &    0.0 &    NaN &    0.0 &   0.00 &   0.00 &    0.0 &    NaN &    0.0 &    0.0 &    NaN &    NaN \\
Pneumocystis   &   0.16 &    0.0 &    NaN &    0.2 &   0.00 &   0.33 &    0.0 &    NaN &    0.0 &    0.0 &    NaN &    NaN \\
Pneumonia      &   0.92 &    1.0 &    NaN &    1.0 &   0.25 &   1.00 &    1.0 &    NaN &    1.0 &    1.0 &    NaN &    NaN \\
SARS           &   0.12 &    0.0 &    NaN &    0.0 &   0.00 &   0.50 &    0.0 &    NaN &    0.0 &    0.0 &    NaN &    NaN \\
Staphylococcus &   0.00 &    0.0 &    NaN &    0.0 &   0.00 &   0.00 &    0.0 &    NaN &    0.0 &    0.0 &    NaN &    NaN \\
Streptococcus  &   0.04 &    0.0 &    NaN &    0.0 &   0.00 &   0.00 &    0.0 &    NaN &    0.0 &    0.0 &    NaN &    NaN \\
Tuberculosis   &   0.08 &    0.0 &    NaN &    0.0 &   0.25 &   0.00 &    0.0 &    NaN &    0.0 &    0.0 &    NaN &    NaN \\
Varicella      &   0.04 &    0.0 &    NaN &    0.0 &   0.00 &   0.00 &    0.0 &    NaN &    0.0 &    0.0 &    NaN &    NaN \\
\bottomrule
\end{tabular}}
\caption{Average characterization of the true negatives (tn), false positives (fp), false negatives (fn), and true positives (tp) of each of the 3 clusters C0, C1, and C2. We compute the characterization by averaging the 20-dimensional ground-truth label of each sample (that describes the pathologies associated with that sample) across all samples belonging in that group. For groups where no sample is assigned, the characterization is NaN.}
\label{tab:covid_labels}
\end{table*}